\documentclass[11pt]{article}
\usepackage{ifthen}
\usepackage{mdwlist}
\usepackage{amsmath,amssymb,amsfonts,amsthm}
\usepackage{bm}
\usepackage{hyperref}
\usepackage{enumitem}
\usepackage{graphicx}
\usepackage{xspace}
\usepackage{verbatim}
\usepackage{algorithm}
\usepackage{algpseudocode}
\usepackage[margin=1in]{geometry}
\usepackage{color}
\usepackage{thmtools}
\usepackage{thm-restate}
\usepackage{multicol}
\usepackage{multirow}
\usepackage[T1]{fontenc}
\usepackage{latexsym}
\usepackage{epsfig}
\usepackage{bbm}

\allowdisplaybreaks

\theoremstyle{plain}
\newtheorem{theorem}{Theorem}[section]     %Restart theorem numberings upon each new section.
\newtheorem{lemma}[theorem]{Lemma}    %The lemma numberings will follow those of theorems.
\newtheorem{corollary}{Corollary}[section]

\theoremstyle{definition}
\newtheorem{definition}{Definition}[section]

\newenvironment{prevproof}[2]{\noindent {\em {Proof of {#1}~\ref{#2}:}}}{$\hfill\qed$\vskip \belowdisplayskip}

\date{}

\begin{document}

\title{Square Hellinger Subadditivity for Bayesian Networks\\ and its Applications to Identity Testing}

\author {
Constantinos Daskalakis\thanks{Supported by a Microsoft Research Faculty Fellowship, and NSF Award CCF-1551875, CCF-1617730 and CCF-1650733.}\\
EECS, MIT \\
\tt{costis@mit.edu}
\and
Qinxuan Pan\thanks{Supported by NSF Award CCF-1551875, CCF-1617730 and CCF-1650733.} \\
EECS, MIT\\
\tt{qinxuan@mit.edu}
}

\maketitle

\begin{abstract}
We show that the square Hellinger distance between two Bayesian networks on the same directed graph, $G$, is subadditive with respect to the neighborhoods of $G$. Namely, if $P$ and $Q$ are the probability distributions defined by two Bayesian networks on the same DAG, our inequality states that the square Hellinger distance, $H^2(P,Q)$, between $P$ and $Q$ is upper bounded by the sum, $\sum_v H^2(P_{\{v\} \cup \Pi_v}, Q_{\{v\} \cup \Pi_v})$, of the square Hellinger distances between the marginals of $P$ and $Q$ on every node $v$ and its parents $\Pi_v$ in the DAG. Importantly, our bound does not involve the conditionals but the marginals of $P$ and $Q$. We derive a similar inequality for more general Markov Random Fields.

As an application of our inequality, we show that distinguishing whether two Bayesian networks $P$ and $Q$ on the same (but potentially unknown) DAG satisfy $P=Q$ vs $d_{\rm TV}(P,Q)>\epsilon$ can be performed from $\tilde{O}(|\Sigma|^{3/4(d+1)} \cdot n/\epsilon^2)$ samples, where $d$ is the maximum in-degree of the DAG and $\Sigma$ the domain of each variable of the Bayesian networks. If $P$ and $Q$ are defined on potentially different and potentially unknown trees, the sample complexity becomes $\tilde{O}(|\Sigma|^{4.5} n/\epsilon^2)$, whose dependence on $n, \epsilon$ is optimal up to logarithmic factors. Lastly, if $P$ and $Q$ are product distributions over $\{0,1\}^n$ and $Q$ is known, the sample complexity becomes $O(\sqrt{n}/\epsilon^2)$, which is optimal up to constant factors.
%
%The running times of the afore-mentioned testing algorithms are polynomial in $n^{d_{\max}}$ for DAGs and polynomial in $n$ for trees and product measures. If the DAG is known the running times are polynomial in $n$ and $|\Sigma|^{d_\max}$.
\end{abstract}

\section{Introduction} \label{sec:intro}

At the heart of scientific activity lies the practice of formulating models about observed phenomena, and developing tools to test the validity of these models. Oftentimes, the models are probabilistic; for example, one may model the effectiveness of a drug in a population as a truncated Normal, or the waiting times in a queuing system as exponential random variables. When a model is probabilistic, testing its validity becomes a distribution testing problem.  In our drug example, one would like to measure the effectiveness of the drug in a sample of the population, and somehow determine whether these samples are ``consistent'' with a truncated Normal distribution. As humans delve into the study of more and more complex phenomena, they quickly face high-dimensional distributions. The goal of this paper is to advance our understanding of high-dimensional hypothesis testing.

Consider the task of testing whether a high-dimensional distribution $P$, to which we have sample access, is identical to some model distribution $Q \in \Delta(\Sigma^n)$, where $\Sigma$ is some alphabet and $n$ is the dimension. A natural goal, which we will call {\em goodness-of-fit testing} in the tradition of Statistics, is to distinguish 
$$P=Q~~~~\text{ from }~~~~d(P,Q)>\epsilon,$$
where $d(\cdot,\cdot)$ is some distance between distributions and $\epsilon >0$ some accuracy parameter. In this paper, we will take $d(\cdot,\cdot)$ to be the total variation distance, although all our results hold if one considers Hellinger distance instead.

Sometimes we do not have a model distribution $Q$, but sample access to two distributions $P, Q \in \Delta(\Sigma^n)$, and we want to determine if they are equal. Again, a natural goal is to distinguish 
$$P=Q~~~~\text{ from }~~~~d(P,Q)>\epsilon.$$
We will call this latter problem, where both distributions are unknown, {\em identity testing}.

As our access to $P$ or to both $P$ and $Q$ in the above problems is via samples, we cannot hope to always solve them correctly. So our goal is actually probabilistic. We want to be correct with probability at least $1-\delta$, for some parameter $\delta$. For ease of presentation, let us take $\delta=1/3$ for the remainder of this paper. Clearly, this probability can be then boosted to arbitrary $\delta$'s at a cost of a factor of $O(\log 1/\delta)$ in the sample complexity.

Both goodness-of-fit and identity testing have received tremendous attention in Statistics. In the above formulation of these problems, they have received a fair amount of attention over the last decade in both Theoretical Computer Science and Information Theory; see e.g. \cite{BatuFFKRW01,BatuKR04,Paninski08,ValiantV14,AcharyaDK15,CDGR15} and their references. Despite intense research, the high-dimensional (large $n$) version of the problems has received much smaller attention~\cite{BatuFFKRW01,Alon2007testing,RubinfeldX10,BhattacharyyaFRV11,AcharyaDK15}, despite its importance for applications. In part, this is due to the fact that the problem, as stated above, is hopeless for large $n$. For example, if $Q$ is the uniform distribution over $\{0,1\}^n$, it is known that $\Theta(2^{n/2}/\epsilon^2)$ samples are necessary (and sufficient) for goodness-of-fit testing~\cite{BatuFFKRW01,Paninski08,ValiantV14}. 

Our goal in this paper is to leverage combinatorial structure in the specification of $P$ and $Q$ to get around these exponential lower bounds. We are motivated by prior work of Daskalakis, Dikkala and Kamath~\cite{DaskalakisDK16}, which initiated the study of testing problems for structured distributions. They considered testing problems for Ising models, showing that goodness-of-fit and independence testing (testing if an Ising model is a product measure over $\{0,1\}^n$) can be solved efficiently from ${\rm poly}(n/\epsilon)$ samples. Their bounds hold for Ising models defined on arbitrary graphs, and for the stronger notion of symmetric Kullback-Leibler divergence (which upper bounds (the square of) total variation distance). In particular, their results are able to side-step the afore-described exponential lower bounds for a broad and important class of probability distributions.

Motivated by this recent work on the Ising model, in this paper we study testing problems on {\em Bayesian networks}, which is a versatile and widely used  probabilistic framework  for modeling high-dimensional distributions with structure. A Bayesian network specifies a probability distribution in terms of a DAG $G$ whose nodes $V$ are random variables taking values in some alphabet $\Sigma$. To describe the probability distribution, one specifies conditional probabilities $P_{X_v|X_{\Pi_v}}(x_v|x_{\Pi_v})$, for all vertices $v$ in $G$, and configurations $x_v\in \Sigma$ and $x_{\Pi_v} \in \Sigma^{\Pi_v}$, where $\Pi_v$ represents the set of parents of $v$ in $G$, taken to be $\emptyset$ if $v$ has no parents. In terms of these conditional probabilities, a probability distribution over $\Sigma^V$ is defined as follows:
$$P(x)  = \prod_{v} P_{X_v | X_{\Pi_v}} (x_v | x_{\Pi_v}), \text{for all }x \in \Sigma^V.$$
 
A special case of a Bayesian network is, of course, a Markov chain, where the graph $G$ is a directed line graph. But Bayesian networks are much more versatile and are in fact universal. They can interpolate between product measures and arbitrary distributions over $\Sigma^V$ as the DAG becomes denser and denser. Because of their versatility they have found myriad applications in diverse fields of application and study, ranging from probability theory to engineering, computational biology, and law. Our goal is to determine whether the basic tasks of goodness-of-fit and identity testing for these fundamental distributions are actually testable. To achieve this, we develop a deeper understanding into the statistical distance between Bayesian networks.

\paragraph{Results and Techniques.} Given sample access to two Bayes nets $P$ and $Q$ on $n$ variables taking values in some set $\Sigma$, we would like to decide whether $P=Q$ vs $\delta(P,Q) \ge \epsilon$, where $\delta(P,Q)$ denotes the total variation distance between $P$ and $Q$. To build our intuition, suppose that $P$ and $Q$ are defined on the same DAG, and $Q$ is given. Our goal is to test the equality of $P$ and $Q$, with fewer than $O(|\Sigma|^{n/2}/\epsilon^2)$ samples required by standard methods, by exploiting the structure of the DAG.

A natural way to exploit the structure of the DAG is to try to ``localize the distance'' between $P$ and $Q$. It is not hard to prove that the total variation distance between $P$ and $Q$ can be upper bounded as follows:
$$\delta(P,Q) \le \sum_v \delta(P_{\{v\}\cup \Pi_v},Q_{\{v\}\cup \Pi_v})+\sum_v \delta(P_{\Pi_v},Q_{\Pi_v}),$$
where, as above, $\Pi_v$ denotes the parents of $v$ in the DAG, if any. The sub-additivity of total variation distance with respect to the neighborhoods of the DAG allows us to argue the following:
\begin{itemize}
\item If $P=Q$, then $P_{\{v\}\cup \Pi_v} = Q_{\{v\}\cup \Pi_v}$, for all $v$.
\item If $\delta(P,Q)\ge \epsilon$, then there exists some $v$ such that $\delta(P_{\{v\}\cup \Pi_v},Q_{\{v\}\cup \Pi_v}) \ge \epsilon/2n$.
\end{itemize}
In particular, we can distinguish between $P=Q$ and $\delta(P,Q) \ge \epsilon$ by running $n$ tests, distinguishing $P_{\{v\}\cup \Pi_v} = Q_{\{v\}\cup \Pi_v}$ vs $\delta(P_{\{v\}\cup \Pi_v},Q_{\{v\}\cup \Pi_v}) \ge \epsilon/2n$, for all $v$. We output ``$P=Q$'' if and only if all these tests output equality. Importantly the distributions $P_{\{v\}\cup \Pi_v}$ and $Q_{\{v\}\cup \Pi_v}$ are supported on $|{\{v\}\cup \Pi_v}|$ variables. Hence if our DAG has maximum in-degree $d$, each of these tests requires $O({|\Sigma|^{(d+1)/2} n^2/\epsilon^2})$ samples. An extra $O(\log n)$ factor in the sample complexity can guarantee that each test succeeds with probability at least $1-1/3n$, hence all tests succeed simultaneously with probability at least $2/3$. Unfortunately the quadratic dependence of the sample complexity on $n$ is sub-optimal.

A natural approach to improve the sample complexity is to consider instead the {\em Kullback-Leibler divergence} between $P$ and $Q$. Pinsker's inequality gives us that ${\rm KL}(P||Q) \ge 2 \delta^2(P,Q)$. Hence, ${\rm KL}(P||Q) =0$, if $P=Q$, while ${\rm KL}(P||Q) \ge 2 \epsilon^2$, if $\delta(P,Q) \ge \epsilon$. Moreover, we can exploit the chain rule of the Kullback-Leibler divergence to  argue the following:
\begin{itemize}
\item If $P=Q$, then $P_{\{v\}\cup \Pi_v} = Q_{\{v\}\cup \Pi_v}$, for all $v$.
\item If $\delta(P,Q) \ge \epsilon$, then there exists some $v$ such that ${\rm KL}(P_{\{v\}\cup \Pi_v} || Q_{\{v\}\cup \Pi_v}) \ge 2\epsilon^2/n$.
\end{itemize}
Hence, to distinguish $P=Q$ vs $\delta(P,Q) \ge \epsilon$ it suffices to run $n$ tests, distinguishing $P_{\{v\}\cup \Pi_v} = Q_{\{v\}\cup \Pi_v}$ vs ${\rm KL}(P_{\{v\}\cup \Pi_v} || Q_{\{v\}\cup \Pi_v}) \ge 2\epsilon^2/n$, for all $v$. We output ``$P=Q$'' if and only if all these tests output equality. Unfortunately, goodness-of-fit with respect to the Kullback-Leibler divergence requires infinitely many samples. On the other hand, if every element in the support of $Q_{\{v\}\cup \Pi_v}$ has probability $\Omega\left({\epsilon^2/n \over |\Sigma|^{d+1}}\right)$, it follows from the $\chi^2$-test of~\cite{AcharyaDK15} that  $P_{\{v\}\cup \Pi_v} = Q_{\{v\}\cup \Pi_v}$ vs ${\rm KL}(P_{\{v\}\cup \Pi_v} || Q_{\{v\}\cup \Pi_v}) \ge 2\epsilon^2/n$ can be distinguished from $O(|\Sigma|^{{d+1 \over 2}}n/\epsilon^2)$ samples. An extra $O(\log n)$ factor in the sample complexity can guarantee that each test succeeds with probability at least $1-1/3n$, hence all tests succeed simultaneously with probability at least $2/3$. So we managed to improve the sample complexity by a factor of $n$. This requires, however, preprocessing the Bayes-nets so that there are no low-probability elements in the supports of the marginals. We do not know how to do this pre-processing unfortunately.

So, to summarize, total variation distance is subadditive in the neighborhoods of the DAG, resulting in $O(n^2/\epsilon^2)$ sample complexity. Kullback-Leibler  is also subadditive and importantly bounds the  square of total variation distance. This is a key to a $O(n/\epsilon^2)$ sample complexity, but it requires no low probability elements in the support of the marginals, which we do not know how to enforce. Looking for a middle ground to address these issues, we study {\em Hellinger distance}, which relates to total variation distance and Kullback-Leibler as follows:
$$\delta(P,Q) \le \sqrt{2} \cdot H(P,Q) \le \sqrt{{\rm KL}(P||Q)}.$$
One of our main technical contributions is to show that the square Hellinger distance between two Bayesian networks on the same DAG is subadditive on the neighborhoods, namely:
\begin{align}
H^2(P,Q) \le \sum_v H^2(P_{\{v\}\cup \Pi_v},Q_{\{v\}\cup \Pi_v}). \label{eq:subadditivity of Hell}
\end{align}
The above bound, given as Corollary~\ref{bayesnetlocalization}, follows from a slightly more general statement given in Section~\ref{hellinger} as Theorem~\ref{subadditivity}. Given the sub-additivity of the Hellinger distance and its relation to total variation, we can follow the same rationale as above to argue the following:
\begin{itemize}
\item If $P=Q$, then $P_{\{v\}\cup \Pi_v} = Q_{\{v\}\cup \Pi_v}$, for all $v$.
\item If $\delta(P,Q) \ge \epsilon$, then there exists some $v$ such that $H^2(P_{\{v\}\cup \Pi_v}, Q_{\{v\}\cup \Pi_v}) \ge \epsilon^2/2n$.
\end{itemize}
Hence, to distinguish $P=Q$ vs $\delta(P,Q) \ge \epsilon$ it suffices to run $n$ tests, distinguishing $P_{\{v\}\cup \Pi_v} = Q_{\{v\}\cup \Pi_v}$ vs $H^2(P_{\{v\}\cup \Pi_v} || Q_{\{v\}\cup \Pi_v}) \ge \epsilon^2/2n$, for all $v$. Importantly goodness-of-fit testing with respect to the square Hellinger distance can be performed from $O(n/\epsilon^2)$ samples. This is the key to our testing results. 

While we presented our intuition for goodness-of-fit testing and when the structure of the Bayes-nets is known, we actually do not need to know the structure and can handle sample access to both distributions. Our results are summarized below. All results below hold if we replace total variation with Hellinger distance.
\begin{itemize}
\item Given sample access to two Bayes-nets $P, Q$ on the same but unknown structure of maximum in-degree $d$, ${\tilde O}(|\Sigma|^{3/4(d+1)} \cdot {n  \over \epsilon^2})$ samples suffice to test $P=Q$ vs $\delta(P,Q) \ge \epsilon$. See Theorem~\ref{thm:identity Bayesnets}. The running time is quasi-linear in the sample size times $O(n^{d+1})$. If the DAG is known, the running time is quasi-linear in the sample size times $O(n)$.

\item Given sample access to two Bayes-nets $P, Q$ on possibly different and unknown trees, $\tilde{O}(|\Sigma|^{4.5} \cdot {n \over \epsilon^2})$ samples suffice to test $P=Q$ vs $\delta(P,Q) \ge \epsilon$. See Theorem~\ref{thm:identity testing trees}. The running time is quasi-linear in the sample size times $O(n^{6})$. The dependence of our sample complexity on $n$ and $\epsilon$ is optimal up to logarithmic factors, as shown by~\cite{DaskalakisDK16}, even when one of the two distributions is given explicitly.

Proving this result presents the additional analytical difficulty that two Bayes-nets on different trees have different factorization, hence it is unclear if their square Hellinger distance can be localized to subsets of nodes involving a small number of variables. In Section~\ref{combinatorics}, we prove that given any pair of tree-structured Bayes-nets $P$ and $Q$, there exists a common factorization of $P$ and $Q$ so that every factor involves up to $6$ variables. This implies an useful subadditivity bound for square Hellinger distance into $n$ subsets of $6$ nodes. See Theorem~\ref{twotreeslocalization}, and the underlying combinatorial lemma, Lemma~\ref{ordering}.

\item Finally, our results above were ultimately based on localizing the distance between two Bayes-nets on neighborhoods of small size, as dictated by the  Bayes-net structure. As we have already mentioned, even if the Bayes-nets are known to be trees, and one of the Bayesnets is given explicitly, $O(n/\epsilon^2)$ samples are necessary. Pushing the simplicitly of the problem to the extreme, we consider the case where both $P$ and $Q$ are Bayes-nets on the empty graph,  $Q$ is given, and $\Sigma=\{0,1\}$. Using a non-localizing test, we show that the identity of $P$ and $Q$ can be tested from $O(\sqrt{n}/\epsilon^2)$ samples, which is optimal up to constant factors, as shown by~\cite{DaskalakisDK16}. See Theorem~\ref{thm:goodness of fit testing product measures}. 

The proof of this theorem also exploits the subadditivity of the square Hellinger distance. Suppose $p_1,\ldots,p_n$ and $q_1,\ldots,q_n$ are the expectations of the marginals of $P$ and $Q$ on the different coordinates, and without loss of generality suppose that $q_i \le {1 \over 2}$, for all $i$. We use the subadditivity of square Hellinger to show that, if $\delta(P,Q) \ge \epsilon$, then $\sum_i {(p_i-q_i)^2 \over q_i} \ge \epsilon^2/2$. Noticing that $\sum_i {(p_i-q_i)^2 \over q_i}$ is an identical expression to the $\chi^2$ divergence applied to vectors $(p_1,\ldots,p_n)$ and $(q_1,\ldots,q_n)$, we reduce the problem to a $\chi^2$-test, mimicking the approach of~\cite{AcharyaDK15}. We only need to be careful that $\sum_i p_i$ and $\sum_i q_i$ do not necessarily equal $1$, but this does not create any issues.
\end{itemize}

\paragraph{Learning vs Testing.} A natural approach to testing the equality between two Bayes-nets $P$ and $Q$ is to first use samples from $P$ and $Q$ to learn Bayes nets $\hat{P}$ and $\hat{Q}$ that are respectively close to $P$ and $Q$, then compare $\hat{P}$ and $\hat{Q}$ offline, i.e. without drawing further samples from $P$ and $Q$. While this approach has been used successfully for single-dimensional hypothesis testing, see e.g.~\cite{AcharyaDK15}, it presents analytical and computational difficulties in the high-dimensional regime. While learning of Bayes nets has been a topic of intense research, including the celebrated Chow-Liu algorithm for tree-structured Bayes-nets~\cite{ChowL68}, we are aware of no computationally efficient algorithms that operate with $\tilde{O}(n/\epsilon^2)$ samples without assumptions. In particular, using net-based techniques~\cite{DevroyeL01,DaskalakisK14,AcharyaJOS14b}, standard calculations show that any Bayes-net on $n$ variables and maximum in-degree $d$ can be learned from $\tilde{O}({n \cdot |\Sigma|^d \over \epsilon^2})$ samples, but this algorithm is highly-inefficient computationally (exponential in $n$). Our algorithms are both efficient, and beat the sample complexity of this inefficient algorithm. On the efficient algorithms front, we are only aware of efficient algorithms that provide guarantees when the number of samples is $>>{n \cdot |\Sigma|^d \over \epsilon^2}$ or that place assumptions on the parameters or the structure of the Bayes-net to be able to learn it (see e.g.~\cite{anandkumar2012high,Bresler15} and their references), even when the structure is a tree~\cite{ChowL68}. Our algorithms do not need any assumptions on the parameters or the structure of the Bayes-net.

\paragraph{Roadmap.} In Section~\ref{hellinger} we present our proof of the square Hellinger subadditivity for a general Markov structures, as Theorem~\ref{subadditivity}. We give corollaries of this theorem to product measures, Markov Chains, tree structured Bayes-nets, and general Bayes-nets. Section~\ref{combinatorics} presents our combinatorial result that two tree-structured Bayes-nets on different graphs always have a common factorization, whose factors only involve up to $6$ variables. Lastly, Section~\ref{testing} presents all our testing results.

\section{Localization Using Hellinger}\label{hellinger}

We first define the Hellinger distance and its square.

\begin{definition}[\textbf{Hellinger}]
For two discrete distributions $p = (p_1, \ldots, p_K)$ and $q = (q_1, \ldots, q_K)$ over a domain of size $K$, their \textit{Hellinger distance} is defined as
\[ H(p,q) = \frac{1}{\sqrt{2}} \sqrt{\sum_{k=1}^K (\sqrt{p_k} - \sqrt{q_k})^2}. \]
The squared Hellinger distance is therefore
\[ H^2(p,q) = \frac{1}{2} \sum_{k=1}^K (\sqrt{p_k} - \sqrt{q_k})^2 = 1 - \sum_{k=1}^K \sqrt{p_k q_k}. \]
\end{definition}

The Hellinger distance always takes value in $[ 0,1 ]$. Compared with the total variation distance $\delta(p,q) = \frac{1}{2} \sum_{k=1}^K |p_k - q_k|$, Hellinger distance satisfies the following inequalities:
\[ H^2(p,q) \leq \delta(p,q) \leq \sqrt{2} H(p,q). \]

We now introduce our main techinical tool in full generality, showing that the squared Hellinger distance is subadditive across components in the factorization of the distributions into products of conditional probability distributions.

\begin{theorem}[\textbf{Squared Hellinger Subadditivity}]
\label{subadditivity}
Let $X = \{ X_1, \ldots, X_n \}$ be a set of random variables that is partitioned disjointedly into a set of super random variables $X_{S_1}, \ldots, X_{S_L}$, where $X_{S_l} = \{X_i\}_{i \in S_l}$. Suppose that $P$ and $Q$ are joint distributions on the $n$ variables with common factorization structure
\begin{align*}
P(x) & = P_{X_{S_1}} (x_{S_1}) \prod_{l=2}^L P_{X_{S_l} | X_{\Pi_l}} (x_{S_l} | x_{\Pi_l}), \\
Q(x) & = Q_{X_{S_1}} (x_{S_1})\prod_{l=2}^L Q_{X_{S_l} | X_{\Pi_l}} (x_{S_l} | x_{\Pi_l}),
\end{align*}
where $\Pi_l \subset S_1 \cup \cdots \cup S_{l-1}$ corresponds to the set of variables conditioned on which $X_{S_l}$ is independent from everything else in the previous super variables. Then
\begin{multline*}
H^2(P,Q) \leq H^2(P_{X_{S_1}}, Q_{X_{S_1}}) + H^2(P_{X_{S_2}, X_{\Pi_2}}, Q_{X_{S_2}, X_{\Pi_2}}) + \cdots + H^2(P_{X_{S_L}, X_{\Pi_L}}, Q_{X_{S_L}, X_{\Pi_L}}),
\end{multline*}
where we use $P$ and $Q$ with subscripts to represent their marginalizations onto the corresponding set of variables. 
\end{theorem}
\begin{proof}
We first prove a simple case. Suppose $P$ and $Q$ are joint distributions on $(X,Y,Z)$ with Markov structure $X \to Y \to Z$, so that
\begin{align*}
& P(x,y,z) = P_X(x)P_{Y|X}(y|x)P_{Z|Y}(z|y),\\
& Q(x,y,z) = Q_X(x)Q_{Y|X}(y|x)Q_{Z|Y}(z|y).
\end{align*}
Then we have
\[ H^2(P, Q) \leq H^2(P_{X,Y}, Q_{X,Y}) + H^2(P_{Y,Z}, Q_{Y,Z}). \]

To show this, consider the following chain of (in)equalities (where we suppressed subscripts when it is clear):
\begin{align*}
H^2(P,Q) & = 1 - \sum_{x,y,z} \sqrt{P(x,y,z)Q(x,y,z)} \\
                & = 1 - \sum_{x,y,z} \sqrt{P(x,y)P(z|y)Q(x,y)Q(z|y)} \\
                & = 1 - \sum_{x,y} \sqrt{P(x,y)Q(x,y)} \sum_z \sqrt{P(z|y)Q(z|y)} \\
                & = 1 - \sum_{x,y} \frac{P(x,y)+Q(x,y)}{2} \sum_z \sqrt{P(z|y)Q(z|y)} \\
                & \qquad + \sum_{x,y} \Bigl(\frac{P(x,y)+Q(x,y)}{2} - \sqrt{P(x,y)Q(x,y)} \Bigr) \sum_z \sqrt{P(z|y)Q(z|y)} \\
                & = 1 - \sum_{y} \frac{P(y)+Q(y)}{2} \sum_z \sqrt{P(z|y)Q(z|y)} \\
                & \qquad + \sum_{x,y} \frac{1}{2} \Bigl(\sqrt{P(x,y)} - \sqrt{Q(x,y)} \Bigr)^2 \underbrace{\sum_z \sqrt{P(z|y)Q(z|y)}}_{\leq 1 \text{ by Cauchy Schwarz}} \\
                & \leq 1 - \sum_{y} \sqrt{P(y)Q(y)} \sum_z \sqrt{P(z|y)Q(z|y)} + \sum_{x,y} \frac{1}{2} \Bigl(\sqrt{P(x,y)} - \sqrt{Q(x,y)} \Bigr)^2  \\
                & = 1 - \sum_{y,z} \sqrt{P(y,z)Q(y,z)} + \frac{1}{2} \sum_{x,y} \Bigl(\sqrt{P(x,y)} - \sqrt{Q(x,y)} \Bigr)^2  \\
                & = 1 - (1 - H^2(P_{Y,Z}, Q_{Y,Z})) + H^2(P_{X,Y}, Q_{X,Y}) \\
                & = H^2(P_{X,Y}, Q_{X,Y}) + H^2(P_{Y,Z}, Q_{Y,Z}).
\end{align*}

Proving the theorem for general $P$ and $Q$ entails applying the simple case repeatedly. First consider the three-node Markov chain
\[ \{X_i\}_{i \in (S_1 \cup \cdots \cup S_{L-1}) \setminus \Pi_L} \to X_{\Pi_L} \to X_{S_L}. \]
We apply the simple case to get
\[ H^2(P,Q) \leq H^2( P_{X_{S_1}, \ldots, X_{S_{L-1}}}, Q_{X_{S_1}, \ldots, X_{S_{L-1}}}) + H^2(P_{X_{S_L}, X_{\Pi_L}}, Q_{X_{S_L}, X_{\Pi_L}}). \]
Next consider the three-node Markov chain
\[ \{X_i\}_{i \in (S_1 \cup \cdots \cup S_{L-2}) \setminus \Pi_{L-1}} \to X_{\Pi_{L-1}} \to X_{S_{L-1}}. \]
We can similarly get
\begin{multline*}
H^2( P_{X_{S_1}, \ldots, X_{S_{L-1}}}, Q_{X_{S_1}, \ldots, X_{S_{L-1}}}) \leq H^2( P_{X_{S_1}, \ldots, X_{S_{L-2}}}, Q_{X_{S_1}, \ldots, X_{S_{L-2}}}) \\ + H^2(P_{X_{S_{L-1}}, X_{\Pi_{L-1}}}, Q_{X_{S_{L-1}}, X_{\Pi_{L-1}}}),
\end{multline*}
If we continue this process and assemble everything at the end, we obtain
\begin{multline*}
H^2(P,Q) \leq H^2(P_{X_{S_1}}, Q_{X_{S_1}}) + H^2(P_{X_{S_2}, X_{\Pi_2}}, Q_{X_{S_2}, X_{\Pi_2}}) + \cdots + H^2(P_{X_{S_L}, X_{\Pi_L}}, Q_{X_{S_L}, X_{\Pi_L}}),
\end{multline*}
proving the general case.
\end{proof}

The subadditivity of the squared Hellinger distance across components immediately allows us to localize the discrepancy onto one component, a result that is crucial to our efficient identity tests for structured high-dimensional distributions in Section~\ref{testing}.

\begin{theorem}[\textbf{Localization}]
\label{localization}
Using the same notation as in Theorem~\ref{subadditivity}, if $P$ and $Q$ satisfy $H^2(P,Q) \geq \epsilon$, then there exists some $l$ such that 
\[ H^2(P_{X_{S_l}, X_{\Pi_l}}, Q_{X_{S_l}, X_{\Pi_l}}) \geq \frac{\epsilon}{L}. \]
\end{theorem}

In the rest of this section, we state several interesting special cases of our results. We first consider when $P$ and $Q$ are product distributions, recovering the well-known result that squared Hellinger is subadditive across individual variables.

\begin{corollary}[\textbf{Product}] \label{cor:Hellinger subadditivity product measures}
Suppose $P$ and $Q$ are joint distributions on $X = \{ X_1, \ldots, X_n \}$ that factorize completely:
\[ P(x) = \prod_{i=1}^n P_{X_i} (x_i), \qquad Q(x) = \prod_{i=1}^n Q_{X_i} (x_i). \]
Then
\[ H^2(P,Q) \leq H^2(P_{X_1}, Q_{X_1}) + \cdots + H^2(P_{X_n}, Q_{X_n}). \]
In particular, if $H^2(P,Q) \geq \epsilon$, then there exists some $i$ such that
\[ H^2(P_{X_i}, Q_{X_i}) \geq \frac{\epsilon}{n}. \]
\end{corollary}

Next, we consider when $P$ and $Q$ have a common Markov chain structure, or a common tree graphical model structure, for which we can localize the discrepancy onto an edge.

\begin{corollary}[\textbf{Markov Chain}] \label{cor:Hellinger subadditivity MC}
Suppose $P$ and $Q$ are joint distributions on the Markov chain $X = X_1 \to X_2 \to \cdots \to X_n$:
\begin{align*}
P(x) & = P_{X_1} (x_1) \prod_{i=2}^n P_{X_i | X_{i-1}} (x_i | x_{i-1}), \\
Q(x) & = Q_{X_1} (x_1) \prod_{i=2}^n Q_{X_i | X_{i-1}} (x_i | x_{i-1}).
\end{align*}
Then
\begin{multline*}
H^2(P,Q) \leq H^2(P_{X_1}, Q_{X_1})  + H^2(P_{X_1,X_2}, Q_{X_1,X_2}) + \cdots + H^2(P_{X_{n-1},X_n}, Q_{X_{n-1},X_n}).
\end{multline*}
In particular, if $H^2(P,Q) \geq \epsilon$, then there exists some $i$ such that
\[ H^2(P_{X_{i-1},X_i}, Q_{X_{i-1},X_i}) \geq \frac{\epsilon}{n}. \]
\end{corollary}

\begin{corollary}[\textbf{Tree}] \label{cor:Hellinger subadditivity trees}
Suppose $P$ and $Q$ are joint distributions on $X = \{ X_1, \ldots, X_n \}$ with a common tree structure:
\begin{align*}
P(x) & = P_{X_1} (x_1) \prod_{i=2}^n P_{X_i | X_{\pi_i}} (x_i | x_{\pi_i}), \\
Q(x) & = Q_{X_1} (x_1) \prod_{i=2}^n Q_{X_i | X_{\pi_i}} (x_i | x_{\pi_i}),
\end{align*}
where we assume without loss of generality that the tree is rooted at $X_1$ and the nodes are ordered in a breadth first search manner away from the root, with $X_{\pi_i}$ being the parent of $X_i$. Then
\begin{multline*}
H^2(P,Q) \leq H^2(P_{X_1}, Q_{X_1})  + H^2(P_{X_2,X_{\pi_2}}, Q_{X_2,X_{\pi_2}}) + \cdots + H^2(P_{X_n,X_{\pi_n}}, Q_{X_n,X_{\pi_n}}).
\end{multline*}
In particular, if $H^2(P,Q) \geq \epsilon$, then there exists some $i$ such that
\[ H^2(P_{X_i,X_{\pi_i}}, Q_{X_i,X_{\pi_i}}) \geq \frac{\epsilon}{n}. \]
\end{corollary}

Moreover, we have the more general case of Bayesian network.

\begin{corollary}[\textbf{Bayes-net}] \label{cor:Hellinger subadditivity Bayesnets}
\label{bayesnetlocalization}
Suppose $P$ and $Q$ are joint distributions on $X = \{ X_1, \ldots, X_n \}$ with a common Bayes-net structure:
\begin{align*}
P(x) & = P_{X_1} (x_1) \prod_{i=2}^n P_{X_i | X_{\Pi_i}} (x_i | x_{\Pi_i}), \\
Q(x) & = Q_{X_1} (x_1) \prod_{i=2}^n Q_{X_i | X_{\Pi_i}} (x_i | x_{\Pi_i}),
\end{align*}
where we assume the nodes are topologically ordered, and $X_{\Pi_i}$ is the set of parents of $X_i$. Then
\begin{multline*}
H^2(P,Q) \leq H^2(P_{X_1}, Q_{X_1})  + H^2(P_{X_2,X_{\Pi_2}}, Q_{X_2,X_{\Pi_2}}) + \cdots + H^2(P_{X_n,X_{\Pi_n}}, Q_{X_n,X_{\Pi_n}}).
\end{multline*}
In particular, if $H^2(P,Q) \geq \epsilon$, then there exists some $i$ such that
\[ H^2(P_{X_i,X_{\Pi_i}}, Q_{X_i,X_{\Pi_i}}) \geq \frac{\epsilon}{n}. \]
\end{corollary}

Observe that if the Bayes-net structure has in-degree at most $d$, then we can localize the discrepancy onto a subset of at most $d+1$ variables.

Finally, we have the important case in which both $P$ and $Q$ have tree structure, but with respect to different trees. Surprisingly, in this case we still can localize the discrepancy onto a subset of constant size. This will help us design efficient identity tests when the underlying tree structures are unknown and distinct. For this, we need to take a combinatorial detour.

\section{Ordering Nodes For Two Trees}\label{combinatorics}

The main goal in this section is a combinatorial lemma stating that given any two trees on the same set of nodes, there is a way to order the nodes so that each node is "dependent" on only constantly many previous nodes, with respect to both trees. We start with  a definition.

\begin{definition}[\textbf{Dependent Set}]
Suppose we have a tree $\mathcal{T}$ and an ordering of its nodes $X_1, \ldots, X_n$. Let $D_\mathcal{T} (X_i)$, the \textit{Dependent Set} of node $X_i$ with respect to $\mathcal{T}$, be the set of nodes $X_k$, $k < i$, such that the (shortest) path between $X_i$ and $X_k$ in $\mathcal{T}$ does not pass through any other $X_j$ with $j < i$.
\end{definition}

Notice that $D_\mathcal{T} (X_i)$ separates $X_i$ from all the other nodes coming before it. If we regard the nodes as variables, and $\mathcal{T}$ as the underlying tree graphical model, then, conditioning on $D_\mathcal{T} (X_i)$, $X_i$ is independent from all the other variables coming before it. We want those conditioning sets to be small, which motivates the following lemma.

\begin{lemma}[\textbf{Ordering}]
\label{ordering}
Given two trees $\mathcal{T}_P$ and $\mathcal{T}_Q$ on a set $X$ of $n$ nodes, we can order the nodes into $X_1, \ldots, X_n$ so that $D_{\mathcal{T}_P} (X_i) \cup D_{\mathcal{T}_Q} (X_i)$ has cardinality at most $5$, for all $i$.
\end{lemma}
\begin{proof}
We make an auxiliary definition: Given a tree $\mathcal{T}$ on $X$, and any $S \subset X$, $\mathcal{T} \setminus S$ consists of a set of connected components (subtrees, in fact). For each such component $\mathcal{T}'$, we define its \textit{boundary} to be the set of nodes in $S$ adjacent to it.

First, we root $\mathcal{T}_P$ and $\mathcal{T}_Q$, independently and arbitrarily. We then pick the nodes one by one, maintaining the following key invariant after each step $i$: Each component $\mathcal{T}_P'$ of $\mathcal{T}_P \setminus \{ X_1, \ldots, X_i \}$ and $\mathcal{T}_Q'$ of $\mathcal{T}_Q \setminus \{ X_1, \ldots, X_i \}$ has boundary size at most $2$, apart from one exception (aggregated across both $\mathcal{T}_P$ and $\mathcal{T}_Q$), whose boundary size can be $3$.

If this invariant is indeed maintained throughout the picking process, then the lemma is correct. To see why, for each $i$, suppose that $X_{i+1}$ lies in the component $\mathcal{T}_P'$ of $\mathcal{T}_P \setminus \{ X_1, \ldots, X_i \}$, and in the component $\mathcal{T}_Q'$ of $\mathcal{T}_Q \setminus \{ X_1, \ldots, X_i \}$. Then, $D_{\mathcal{T}_P} (X_{i+1})$ is exactly the boundary of $\mathcal{T}_P'$, and $D_{\mathcal{T}_Q} (X_{i+1})$ is exactly the boundary of $\mathcal{T}_Q'$. Therefore, $D_{\mathcal{T}_P} (X_{i+1}) \cup D_{\mathcal{T}_Q} (X_{i+1})$ has size at most $5$.

So we only need to show that the invariant can always be maintained. This can be done inductively. Suppose the invariant holds after picking $X_1, \ldots, X_i$. WLOG, suppose that component $\mathcal{T}_P^*$ of $\mathcal{T}_P \setminus \{ X_1, \ldots, X_i \}$ is the single exception with boundary size allowed up to $3$ (if there is no exception, then just pick $\mathcal{T}_P^*$ to be any component).

Apart possibly from the parent (in $\mathcal{T}_P$) of the root of $\mathcal{T}_P^*$, each node in the boundary of $\mathcal{T}_P^*$ must be a child of some node in $\mathcal{T}_P^*$. Consider the number of such nodes. If there is none, then we can pick $X_{i+1}$ to be any node in $\mathcal{T}_P^*$. If there is one, then we pick $X_{i+1}$ to be its parent. If there are two, then we pick $X_{i+1}$ to be their lowest common ancestor. Finally, if there are three, consider their pairwise lowest common ancestors. We pick $X_{i+1}$ to be the lowest of those. In each case, $\mathcal{T}_P^*$ will be turned into one or more components each with boundary size at most $2$, while the other components of $\mathcal{T}_P \setminus \{ X_1, \ldots, X_i \}$ are undisturbed in passing to $\mathcal{T}_P \setminus \{ X_1, \ldots, X_{i+1} \}$. Moreover, there will be at most one component of boundary size $3$ of $\mathcal{T}_Q \setminus \{ X_1, \ldots, X_{i+1} \}$ (spawned from the component of $\mathcal{T}_Q \setminus \{ X_1, \ldots, X_i \}$ containing $X_{i+1}$). Thus, the invariant is maintained, and the proof of the lemma is complete.
\end{proof}

Using Lemma~\ref{ordering}, we can prove our discrepancy localization result when the two tree-structured distributions have distinct underlying trees.

\begin{corollary}[\textbf{Two Trees}]
\label{twotreeslocalization}
Suppose $P$ and $Q$ are tree-structured joint distributions on a set $X$ of $n$ variables, with possibly distinct underlying trees. Then there exists an ordering of the nodes into $X_1, \ldots, X_n$, and sets $\Pi_i \subset \{1, \ldots , i-1 \}$ of cardinality at most $5$ for all $i$, such that $P$ and $Q$ have the common factorization
\begin{align*}
P(x) & = P_{X_1} (x_1) \prod_{i=2}^n P_{X_i | X_{\Pi_i}} (x_i | x_{\Pi_i}), \\
Q(x) & = Q_{X_1} (x_1) \prod_{i=2}^n Q_{X_i | X_{\Pi_i}} (x_i | x_{\Pi_i}).
\end{align*}
Consequently,
\begin{multline*}
H^2(P,Q) \leq H^2(P_{X_1}, Q_{X_1})  + H^2(P_{X_2,X_{\Pi_2}}, Q_{X_2,X_{\Pi_2}}) + \cdots + H^2(P_{X_n,X_{\Pi_n}}, Q_{X_n,X_{\Pi_n}}).
\end{multline*}
In particular, if $H^2(P,Q) \geq \epsilon$, then there exists some $i$ such that
\[ H^2(P_{X_i,X_{\Pi_i}}, Q_{X_i,X_{\Pi_i}}) \geq \frac{\epsilon}{n}. \]
In other words, we can localize the discrepancy onto a subset of at most $6$ variables.
\end{corollary}
\begin{proof}
Let $\mathcal{T}_P$ and $\mathcal{T}_Q$ be the underlying tree structures of $P$ and $Q$, respectively. We adopt the node ordering obtained from Lemma~\ref{ordering}, and choose $\Pi_i$ to be set of indices corresponding to the nodes in $D_{\mathcal{T}_P} (X_i) \cup D_{\mathcal{T}_Q} (X_i)$, which has cardinality at most $5$ for all $i$.

For each $i$, the nodes $D_{\mathcal{T}_P} (X_i)$ separate $X_i$ from $ \{ X_1, \ldots, X_{i-1} \} \setminus D_{\mathcal{T}_P} (X_i)$ in $\mathcal{T}_P$. Consequently, conditioning on $D_{\mathcal{T}_P} (X_i)$, $X_i$ is independent from $ \{ X_1, \ldots, X_{i-1} \} \setminus D_{\mathcal{T}_P} (X_i)$ for distribution $P$. Since $X_{\Pi_i} \supset D_{\mathcal{T}_P} (X_i)$, we also have that conditioning on $X_{\Pi_i}$, $X_i$ is independent from $ \{ X_1, \ldots, X_{i-1} \} \setminus X_{\Pi_i}$ for $P$. Since this is true for all $i$, $P$ has the desired factorization into a product of conditionals. The argument for $Q$ is completely symmetric.

The rest of the Corollary follows from Theorem~\ref{subadditivity} and Theorem~\ref{localization}.
\end{proof}

\section{Applications to Identity Testing}\label{testing}

In this section, we use the tools developed in Sections~\ref{hellinger} and~\ref{combinatorics} to construct efficient identity tests for structured high-dimensional distributions. Even though we are ultimately interested in total variation distance, the bulk of our manipulations, as well as the tests administered after localization, are in (squared) Hellinger distance. So we first recall a test for (squared) Hellinger distance. 

\begin{lemma}[\textbf{Hellinger Test}~\cite{DiakonikolasK16}]
\label{hellingertest}
Suppose we have sample access to unknown distributions $P$ and $Q$ over the same set of size $D$. Then, from $\tilde{O}\left(\min\left( D^{2/3}/\epsilon^{8/3}, D^{3/4}/\epsilon^2 \right) \right)$ samples from each, we can distinguish between $P = Q$ vs. $H^2(p,q) \geq \epsilon^2$ with error probability at most $1/3$. This probability can be made an arbitrary $\eta$ at a cost of an additional factor of $O(\log 1/\eta)$ in the sample complexity.
\end{lemma}

We now have everything ready for our identity tests. Recall that $\delta$ denotes the total variation distance.

\begin{theorem}[Identity-Testing for \textbf{Bayes-nets}] \label{thm:identity Bayesnets}
Suppose we have sample access to {\em unknown} joint distributions $P$ and $Q$ on variables $X = \{ X_1, \ldots, X_n \}$ with a {\em known and common} Bayes-net structure:
\begin{align*}
P(x) = P_{X_1} (x_1) \prod_{i=2}^n P_{X_i | X_{\Pi_i}} (x_i | x_{\Pi_i}), \\
Q(x) = Q_{X_1} (x_1) \prod_{i=2}^n Q_{X_i | X_{\Pi_i}} (x_i | x_{\Pi_i}),
\end{align*}
where we assume the nodes are topologically ordered, $X_{\Pi_i}$ is the set of parents of $X_i$, and every variable takes values in some alphabet $\Sigma$ of size $K$. Suppose each $\Pi_i$ has size at most $d$. Then, from $\tilde{O}(K^{3/4\cdot(d+1)} \cdot \frac{n}{\epsilon^2})$ samples and in $\tilde{O} (K^{3/4\cdot(d+1)} \cdot \frac{n^2}{\epsilon^2})$ time, we can distinguish between $P=Q$ vs. $\delta(P,Q) \geq \epsilon$ with error probability at most $\frac{1}{3}$.

\medskip Similarly, suppose that we are given sample access to {\em unknown} joint distributions $P$ and $Q$ on a {\em common but unknown} Bayes-net structure whose maximum in-degree $d$ is known. From $\tilde{O} (K^{3/4\cdot(d+1)} \cdot \frac{n}{\epsilon^2})$ samples and in $\tilde{O} (K^{3/4\cdot(d+1)} \cdot \frac{ n^{d+2} }{\epsilon^2})$ time, we can distinguish between $P=Q$ vs. $\delta(P,Q) \geq \epsilon$ with error probability at most $\frac{1}{3}$.
\end{theorem}
\begin{proof}
We first prove the first part of the theorem, where the shared structure of the Bayes-nets is known. For each $i = 1, \ldots, n$, we run the Hellinger test to distinguish between
\[ P_{X_i,X_{\Pi_i}} = Q_{X_i,X_{\Pi_i}} \quad \text{vs.} \quad H^2(P_{X_i,X_{\Pi_i}}, Q_{X_i,X_{\Pi_i}}) \geq \frac{\epsilon^2}{2n}. \]
We return ``$P=Q$'' if and only if each of those $n$ subtests returns equality.

Note that the domain size for each subtest is at most $K^{d+1}$. With the right choice of $O(\log n)$ factor in our sample complexity, Lemma~\ref{hellingertest} implies that each of the $n$ sub-tests has error probability at most $\frac{1}{3n}$. Consequently, with probability at least $\frac{2}{3}$, all those subtests give correct answers (when $H^2(P_{X_i,X_{\Pi_i}}, Q_{X_i,X_{\Pi_i}})$ is strictly between $0$ and $\frac{\epsilon^2}{2n}$, either answer is deemed correct). It suffices to show that our test is correct in this situation.

If $P=Q$, then $P_{X_i,X_{\Pi_i}} = Q_{X_i,X_{\Pi_i}}$ for every $i$. So every subtest will return equality, and we will return ``$P=Q$".

If $\delta(P, Q) \geq \epsilon$, then $H(P,Q) \geq \frac{\epsilon}{\sqrt{2}}$, and so $H^2(P,Q) \geq \frac{\epsilon^2}{2}$. By Corollary~\ref{bayesnetlocalization}, there exists some $i$ such that $H^2(P_{X_i,X_{\Pi_i}}, Q_{X_i,X_{\Pi_i}}) \geq \frac{\epsilon^2}{2n}$. Consequently, the $i$th subtest will not return equality, and we will return ``$\delta (P,Q) \geq \epsilon$".

The running time bound follows from the fact that we perform $n$ Hellinger tests, each of which takes time quasi-linear in the sample size.

\medskip The second part of our theorem follows similarly. From the argumentation above it follows that, if $\delta(P, Q) \geq \epsilon$, then there exists some set $S$ of variables of size $|S|=d+1$ such that $H^2(P_{X_S}, Q_{X_S}) \geq \frac{\epsilon^2}{2n}$. On the other hand, it is obvious that, if $P=Q$, then for all sets $S$: $P_{X_S} = Q_{X_S}$. So we can run the Hellinger test on every set $S$ of $d+1$ variables to distinguish between:
\[ P_{X_S} = Q_{X_S} \quad \text{vs.} \quad H^2(P_{X_S}, Q_{X_S}) \geq \frac{\epsilon^2}{2n}. \]
We return ``$P=Q$'' if and only if all of these tests return equality. Since we are running $O(n^{d+1})$ Hellinger tests, we want that each has success probability $1-n^{(-\Omega(d))}$ to do a union bound. This results in an extra factor of $O(d)$ in the sample complexity compared to the known structure case analyzed above, where we only performed $n$ Hellinger tests. The running time bound follows because we run  $O(n^{d+1})$ tests, each of which takes quasi-linear time in the sample.
\end{proof}

\begin{theorem}[Identity-Testing for \textbf{Tree Structured Bayes-nets}] \label{thm:identity testing trees}
Suppose we have sample access to unknown Bayes-nets $P$ and $Q$ on variables $X = \{ X_1, \ldots, X_n \}$ with {\em unknown, and possibly distinct, tree structures}. Suppose that every variable takes values in some alphabet $\Sigma$ of size $K$.

Then, from $\tilde{O}(K^{4.5} \cdot \frac{n}{\epsilon^2})$ samples and in $\tilde{O}(K^{4.5} \cdot \frac{n^7}{\epsilon^2})$ time, we can distinguish between $P=Q$ vs. $\delta(P,Q) \geq \epsilon$ with error probability at most $\frac{1}{3}$. The dependence of the sample complexity on $n$ and $\epsilon$ is tight up to logarithmic factors. 
\end{theorem}
\begin{proof}
For each subset $X_S$ of at most $6$ variables, we run the Hellinger test to distinguish between
\[ P_{X_S} = Q_{X_S} \quad \text{vs.} \quad H^2(P_{X_S}, Q_{X_S}) \geq \frac{\epsilon^2}{2n}. \]
We return ``$P=Q$" if and only if each of those $\Theta(n^6)$ subtests returns equality.

Note that the domain size for each subtest is at most $K^6$. With the right choice of $O(\log n)$ factor in our sample complexity, Lemma~\ref{hellingertest} implies that each of the sub-tests has error probability at most $\frac{1}{3n^6}$. Consequently, with probability at least $\frac{2}{3}$, all those subtests give correct answers. It suffices to show that our test is correct in this situation.

If $P=Q$, then $P_{X_S} = Q_{X_S}$ for every $X_S$. So every subtest will return equality, and we will return ``$P=Q$".

If $\delta(P, Q) \geq \epsilon$, then $H(P,Q) \geq \frac{\epsilon}{\sqrt{2}}$, and so $H^2(P,Q) \geq \frac{\epsilon^2}{2}$. By Corollary~\ref{twotreeslocalization}, there exists some $X_S$ of size at most $6$ such that $H^2(P_{X_S}, Q_{X_S}) \geq \frac{\epsilon^2}{2n}$. This will be detected by one of the subtests, and we will return ``$\delta (P,Q) \geq \epsilon$".

The running time bound follows from the fact that we perform $\Theta (n^6)$ Hellinger tests, each of which takes time quasi-linear in the sample size.

\medskip The tightness of the sample complexity follows directly from Theorem 18 and Remark 2 of~\cite{DaskalakisDK16}. They show that, given sample access to a ferromagnetic Ising model $P$, which is known to have a tree structure, one needs $\Omega(n/\epsilon^2)$ samples to distinguish whether $P$ equals the uniform distribution over $\{0,1\}^n$ vs $P$ being $\epsilon$-far in total variation distance from the uniform distribution. Since Ising models on trees and Bayes-nets on trees with binary alphabets are equivalent, the lower bound follows.
\end{proof}

\begin{theorem}[Goodness-of-fit for \textbf{Product Distributions over the Hypercube}] \label{thm:goodness of fit testing product measures}
Suppose $P$ and $Q$ are product distributions over $\{0,1\}^n$. We are given sample access to $P$, while $Q$ is known exactly. Then, from $O\left({\sqrt{n} \over \epsilon^2}\right)$ samples and in $O\left({n^{1.5} \over \epsilon^2}\right)$ time, we can distinguish between $P=Q$ vs. $\delta(P,Q) \geq \epsilon$ with error probability at most $\frac{1}{3}$. The sample complexity is optimal up to $O(1)$ factors.
\end{theorem}

\begin{prevproof}{Theorem}{thm:goodness of fit testing product measures}
We start with a useful lemma:

\begin{lemma}\label{lemma:Bernoulli hellinger upper bound}
If $X$ and $Y$ are Bernoulli random variables with expectations $p$ and $q$, then
$$H^2(X,Y) \le {(p-q)^2 \over 2}\left({1 \over q} + {1 \over 1-q}\right).$$
\end{lemma}
\begin{prevproof}{Lemma}{lemma:Bernoulli hellinger upper bound}
Suppose $p=q+x$ for some $x$. From the definition of Hellinger distance we have that:
\begin{align*}
H^2(X,Y)&=1-\sqrt{q \cdot (q+x)} - \sqrt{(1-q) \cdot (1-q-x)}\\
&=1-q \cdot \sqrt{1+{x \over q}} - (1-q) \cdot \sqrt{1-{x \over 1-q}}\\
&\le 1-q \cdot \left(1+{x \over 2q} - {x^2 \over 2 q^2}\right) - (1-q)\cdot \left(1-{x \over 2(1-q)}-{x^2 \over 2(1-q)^2}\right)\\
&= {x^2 \over 2 q} + {x^2 \over 2(1-q)}={x^2 \over 2}\left({1 \over q} + {1 \over 1-q}\right).
%&= 1-q \cdot (1+{x \over 2q} - {x^2 \over 2 q^2}) - (1-q)\cdot (1-{x \over 2(1-q)}+{x^2 \over 2(1-q)^2)
\end{align*}
The inequality in the above derivation follows from the following lemma, whose simple proof is omitted, and the realization that the constraints $p,q \in [0,1]$ and $p=q+x$ imply that ${x \over q} \ge -1$ and $-{x \over 1-q} \ge -1$.
\begin{lemma} For all $t \ge -1$: $\sqrt{1+t} \ge 1+{t \over 2}-{t^2 \over 2}$.
\end{lemma}
\end{prevproof}

Now let us turn to the proof of Theorem~\ref{thm:goodness of fit testing product measures}. Suppose distribution $P$ samples a collection $X=(X_1, \ldots, X_n)$ of mutually independent binary random variables, and $Q$ samples a collection $Y=(Y_1, \ldots, Y_n)$ of mutually independent binary random variables. Suppose that $\mathbf{E}[X_i] = p_i$ and $\mathbf{E}[Y_i] = q_i$. It follows from Corollary~\ref{cor:Hellinger subadditivity product measures} and Lemma~\ref{lemma:Bernoulli hellinger upper bound}
\begin{align}
H^2(P,Q) \equiv H^2(X,Y) \le \sum_i H^2(X_i, Y_i) \le \sum_i {(p_i-q_i)^2 \over 2}\left({1 \over q_i} + {1 \over 1-q_i}\right). \label{eq:hellinger upper bound product distributions over hypercube}
\end{align}

Next, we can assume without loss of generality that for all $i$: $q_i \in [{\epsilon \over c n}, 1/2]$, for some large enough constant $c$ to be chosen later. Indeed:
\begin{itemize}
\item Suppose that some $i$ satisfies $q_i>1/2$. Then we can define distributions $P'$ and $Q'$ that sample $(X_1,\ldots,X_{i-1},1-X_i,X_{i+1},\ldots,X_n)$ and $(Y_1,\ldots,Y_{i-1},1-Y_i,Y_{i+1},\ldots,Y_n)$ respectively where $X \sim P$ and $Y \sim Q$, and test the identity of $P'$ with $Q'$. This suffices since $\delta(P,Q)=\delta(P',Q')$. As this can be done for all coordinates $i$ such that $q_i > 1/2$, without affecting the total variation distance, we can assume without loss of generality that $Q$ satisfies $q_i \in [0,1/2]$ for all $i$. 

\item Now, suppose that some $i$ satisfies $q_i < {\epsilon \over c n}$. We can define distributions $P'$ and $Q'$ that sample $(X_1 \oplus Z_1,\ldots,X_n \oplus Z_n)$ and $(Y_1 \oplus W_1,\ldots,Y_n \oplus W_n)$, where $X \sim P$, $Y \sim Q$, $Z_1, \ldots, Z_n$ are mutually independent Bernoullis that are independent of $X$, $W_1, \ldots, W_n$ are mutually independent Bernoullis that are independent of $Y$, and $\oplus$ is the XOR operation between bits. For all $i$, if $q_i < {\epsilon \over c n}$, both $Z_i$ and $W_i$ are sampled from ${\rm Bernoulli}({\epsilon \over cn})$, otherwise they are sampled from ${\rm Bernoulli}(0)$. Clearly, both $P'$ and $Q'$ remain product distributions over $\{0,1\}^n$. Moreover, for all $i$, ${\bf E}[Y_i \oplus W_i] \ge {\epsilon \over c n}$ (as long as ${2 \epsilon \over c n} \le 1$, which we will assume for large enough $c$). Moreover, it is clear that, if $P=Q$, then $P'=Q'$. It is also easy to see that $\delta(P,P') \le {\epsilon \over c}$ and $\delta(Q,Q') \le {\epsilon \over c}$. To show that $\delta(P,P') \le {\epsilon \over c}$ we need to exhibit a coupling between vectors $X \sim P$ and $X' \sim P'$ such that the probability that $X \neq X'$ under the coupling is at most $\epsilon \over c$. This is trivial to obtain and, in fact, we have already given it. Suppose that $X \sim P$, $Z_1,\ldots,Z_n$ are Bernoulli random variables defined as above, and $X'=(X_1 \oplus Z_1,\ldots,X_n \oplus Z_n)$. Clearly, in the event that $Z_1=\ldots,Z_n=0$, which happens with probability at least $1-{\epsilon \over c}$, $X=X'$, proving that under our coupling the probability that $X\neq X'$ is at most ${\epsilon \over c}$ and establishing $\delta(P,P') \le {\epsilon \over c}$. We can similarly show that $\delta(Q,Q') \le {\epsilon \over c}$. From these and the triangle inequality, it follows that if $\delta(P,Q) \ge \epsilon$, then $\delta(P',Q') \ge \epsilon-{2\epsilon \over c}$. Hence to distinguish between $P=Q$ and $\delta(P,Q)\ge \epsilon$, it suffices to be able to distinguish between $P'=Q'$ and $\delta(P',Q')\ge \epsilon-{2\epsilon \over c} = (1-{2\over c})\epsilon$. So making the assumption $q_i \ge {\epsilon \over cn}$, only degrades by a constant factor the total variation distance between the distributions we need to distinguish.
\end{itemize}
 
Now let us get to the core of our proof. Under our assumption that $q_i \le 1/2$ for all $i$, Eq~\eqref{eq:hellinger upper bound product distributions over hypercube} implies that:
\begin{align}
H^2(P,Q)  \le \sum_i {(p_i-q_i)^2 \over q_i}. \label{eq:hellinger upper bound product distributions over hypercube}
\end{align}
It follows from the above that:
\begin{itemize}
\item if $P=Q$, then $\sum_i {(p_i-q_i)^2 \over q_i}=0$;
\item if $\delta(P,Q)\ge \epsilon$, then $\sum_i {(p_i-q_i)^2 \over q_i} \ge \epsilon^2/2$ (where we used that $\sqrt{2} H(P,Q) \ge \delta(P,Q)$).
\end{itemize}

{\em Observation:} Notice that, if $\sum_i p_i=1$ and $\sum_i q_i=1$, we could interpret $p_1,\ldots, p_n$ as a distribution $\tilde{p}$ sampling $i \in \{1,\ldots,n\}$ with probability $p_i$ and we could similarly define distribution $\tilde{q}$ to sample $i \in \{1,\ldots,n\}$ with probability $q_i$. Then $\sum_i {(p_i-q_i)^2 \over q_i} \equiv \chi^2(\tilde{p},\tilde{q})$. Inspired by this observation, our analysis will mimick the analysis of~\cite{AcharyaDK15}, who provided algorithms for distinguishing whether $\chi^2(\tilde{p},\tilde{q})\le \epsilon^2/c_1$ vs $\chi^2(\tilde{p},\tilde{q})\ge \epsilon^2/c_2$, for constants $c_1,c_2$ (under the assumption $q_i \ge \epsilon/c n$, for some constant $c$). Their algorithm uses $O(\sqrt{n}/\epsilon^2)$ samples, which is our goal here.

\smallskip Inspired by our observation, our algorithm is the following. For $m \ge {c'\sqrt{n} \over \epsilon^2}$, we draw $M_i \sim {\rm Poisson}(m)$ independently for all $i$. Using $\max_i{M_i}$ many samples from $P$, we can obtain $M_i$ independent samples from the $i$-th coordinate of $P$. Let $N_i$ denote how many out of these $M_i$ samples are $1$. The way we have set up our sampling guarantees that, for all $i$, $N_i \sim {\rm Poisson}(m \cdot p_i)$, and moreover the variables $N_1,\ldots,N_n$ are mutually independent. Finally, by the strong concentration of the Poisson distribution, $\max_i M_i \le 2 e m$, with probability at least $1-3\cdot ({1 \over 2})^{2 e c' \sqrt{n} \over \epsilon^2}$. (If $\max_i M_i > 2 e m$, we can have our test output a random guess to avoid asking more than $2 e m$ samples from $P$. As we are shooting for error probability $1/3$ we can accommodate this.)

Let us now form the following statistic, mimicking~\cite{AcharyaDK15}:
$Z=\sum_{i}{(N_i-mq_i)^2-N_i \over m q_i}$. It was shown in~\cite{AcharyaDK15}, that ${\bf E}[Z] = m \sum_i {(p_i-q_i)^2 \over q_i}$. It can easily be checked that this equality is unaffected by the values of $\sum_i p_i$ and $\sum_i q_i$. Hence:
\begin{itemize}
\item if $P=Q$, then ${\bf E}[Z]=0$;
\item if $\delta(P,Q) > \epsilon$, then ${\bf E}[Z] \ge m \epsilon^2/2$.
\end{itemize} 
It was also shown in~\cite{AcharyaDK15}, that
$${\bf Var}[Z] \le 4 n + 9 \sqrt{n} {\bf E}[Z]+{2 \over 5}n^{1/4}{\bf E}[Z]^{3/2}.$$
This bound holds as long as $m \ge c'\sqrt{n}/\epsilon^2$ and $q_i \ge \epsilon/c n$, for all $i$, and $c',c$ are large enough. Again, the above bound remains true even if $\sum_i p_i$ and $\sum_i q_i$ do not equal $1$. Hence:
\begin{itemize}
\item if $P=Q$, then ${\bf E}[Z]=0$, hence ${\bf Var}[Z] \le 4 n \le {4 \over c'^2} m^2 \epsilon^4$;
\item if $\delta(P,Q) > \epsilon$, then ${\bf E}[Z] \ge m \epsilon^2/2 \ge {c' \sqrt{n} \over 2}$, hence ${\bf Var}[Z] \le {16 \over c'^2} {\bf E}[Z]^2+{18 \over c'} {\bf E}[Z]^2 + {2 \sqrt{2} \over 5 \sqrt{c'}}{\bf E}[Z]^2.$
\end{itemize} 
The theorem follows from the above via an application of Chebyshev's inequality, and choosing $c'$ large enough.

\medskip The tightness of our sample complexity follows directly from Theorem 17 and Remark 2 of~\cite{DaskalakisDK16}.
\end{prevproof}

\bibliographystyle{alpha}
\bibliography{biblio}

\end{document}